\documentclass{article}
\usepackage{stfloats}
\usepackage{microtype}
\usepackage{graphicx}
\usepackage{subfigure}
\usepackage{booktabs} 
\usepackage{fancyhdr}
\usepackage{mathtools}
\usepackage[utf8]{inputenc} 
\usepackage[T1]{fontenc}    
\usepackage{hyperref}       
\usepackage{url}            
\usepackage{booktabs}       
\usepackage{amsfonts}       
\usepackage{nicefrac} 
\usepackage{microtype}
\usepackage{graphicx}
\usepackage{authblk}
\usepackage{caption}
\usepackage{subfigure}
\usepackage{amsmath}
\usepackage{multirow}
\usepackage{array}
\usepackage{verbatim}
\usepackage{datetime}

\usepackage{amsmath,amsfonts,bm}

\def\eqref#1{equation~\ref{#1}}

\def\1{\bm{1}}

\DeclareMathAlphabet{\mathsfit}{\encodingdefault}{\sfdefault}{m}{sl}
\SetMathAlphabet{\mathsfit}{bold}{\encodingdefault}{\sfdefault}{bx}{n}

\newcommand{\R}{\mathbb{R}}

\usepackage{hyperref}

\newcommand*{\defeq}{\stackrel{\mathsf{def}}{=}}

\def\defeq{ \stackrel{\triangle}{=}}

\def\cG{ {\mathcal G}}

\def\cL{ {\mathcal L}}
\def\cM{ {\mathcal M}}
\def\cN{ {\mathcal N}}

\def\cS{ {\mathcal S}}

\usepackage[accepted]{icml2021}

\begin{document}

\twocolumn[
\icmltitle{Deep Manifold Transformation for Dimension Reduction and Visualization}



\icmlsetsymbol{equal}{*}

\begin{icmlauthorlist}
\icmlauthor{Stan Z. Li}{wu}
\icmlauthor{Zelin Zang}{wu}
\icmlauthor{Lirong Wu}{wu}\\
\icmlauthor{AI Lab, School of Engineering, Westlake University, China}{wu}
\end{icmlauthorlist}

\icmlaffiliation{wu}{AI Lab, School of Engineering, Westlake University, China}

\icmlcorrespondingauthor{Stan Z. Li}{Stan.ZQ.Li@westlake.edu.cn}

\icmlkeywords{Machine Learning, ICML}
\vskip 0.3in
]




\pagenumbering{arabic}


\title{Deep Manifold Transformation With Cross-Layer Geometry-Preserving Constraints}

\begin{abstract}
Manifold learning-based encoders have been playing important roles in nonlinear dimensionality reduction (NLDR) for data analysis. However, existing methods can often fail to preserve geometric or distributional structures of data. 
In this paper, we propose a deep manifold learning framework, called {\em deep manifold transformation} (DMT) for  unsupervised NLDR and embedding learning. DMT enhances deep neural networks by using cross-layer {\em local geometry-preserving} (LGP) constraints. The LGP constraints constitute the loss for deep manifold learning and serve as geometric regularizers for NLDR network training. Extensive experiments on synthetic and real-world data demonstrate that DMT networks outperform existing leading manifold-based NLDR methods in terms of local and global geometry preservation of data distributions.
\end{abstract}



\section{Introduction}

Encoders have played an important role in {\em nonlinear dimensionality reduction} (NLDR) in analyzing complex high dimensional data. Take the classic deep autoencoder (AE) \cite{hinton2006reducing} for example. It uses an encoder, a multi-layer feed-forward network, to transform the input data to an embedding in a latent space of lower dimensionality; in the meantime, it uses the output of a decoder to compute the reconstruction error as the loss. 

One approach to NLDR is through manifold learning. Such methods are based on the manifold assumption \cite{Belkin-Niyogi-03,Fefferman-manifold-2016}, which states that patterns of interest in high dimensional data are lower-dimensional manifolds residing in the data space. Therefore, NLDR has been extensively studied in the context of {\em manifold learning} \cite{Belkin-Niyogi-03,Tenenbaum-science-00,Roweis-science-00,donoho2003hessian,Gashler-NIPS-2007, zhang2007mlle, CHEN-JASA-2009, McQueen-NIPS-2016,Saul-PNAS-2020}. Preserving the geometric structure of data is an important property to achieve. The geometric structure usually includes two aspects: structure of neighboring points on a manifold and relative locations among different manifolds. The former is related to local structure, whereas the latter global structure. 

Numerous manifold learning-based NLDR methods have been proposed. Isometric Mapping (ISOMAP) \cite{Tenenbaum-science-00} and locally linear embedding (LLE) \cite{Roweis-science-00} are classic ones among others. Later developments include Hessian LLE (HLLE) \cite{donoho2003hessian}, Modified LLE (MLLE) \cite{zhang2007mlle}, topologically constrained isometric embedding (TCIE) \cite{Bronstein-TCIE-2010}, and more recently, latent variable models (LVMs) \cite{Saul-PNAS-2020}. t-Distributed Stochastic Neighbor Embedding (t-SNE) \cite{Maaten-tSNE-2014} and Uniform Manifold Approximation and Projection (UMAP) \cite{mcinnes2018umap} are two popular ones for manifold learning-based NLDR, widely used for NLDR and visualization. Topological autoencoder (TAE) \cite{Moor-TAE-ICML2020}, as a deep learning method,  imposes topological constraints \cite{Wasserman-2018} on top of the autoencoder architecture to preserve the topological structure of data. Deep Isometric MAnifold Learning (DIMAL) \cite{pai2019dimal} combines deep learning framework with multi-dimensional scaling objective, which can be seen as a neural network version of MDS. With sparse geodesic sampling, DIMAL can learn a distance-preserving mapping to generate low-dimensional embeddings for a certain class of manifolds with only a few sampling points. A summary of well known NLDR and manifold learning algorithms can be found in \cite{Wiki-NLDR}. A deep learning approach to manifold-based NLDR has been suggested as ``manifold learning 2.0'' \cite{Bronstein-2020}; however, there have no effective approaches.



The above leading manifold learning methods suffer from one or more of the following problems: 
(1) Although aimed to preserve geometrical or distributional structures, t-SNE and UMAP do not necessarily have fulfilled the promises, and MLLE and TAE have their limitations, too, as shown in Figure~\ref{fig:QualitativeComparison} for some examples.
(2) ISOMAP, LLE, t-SNE, and UMAP do not yield a transformation generalizable to unseen data.


\begin{figure}[!htb]
  \centering{\includegraphics[width=3.3in]{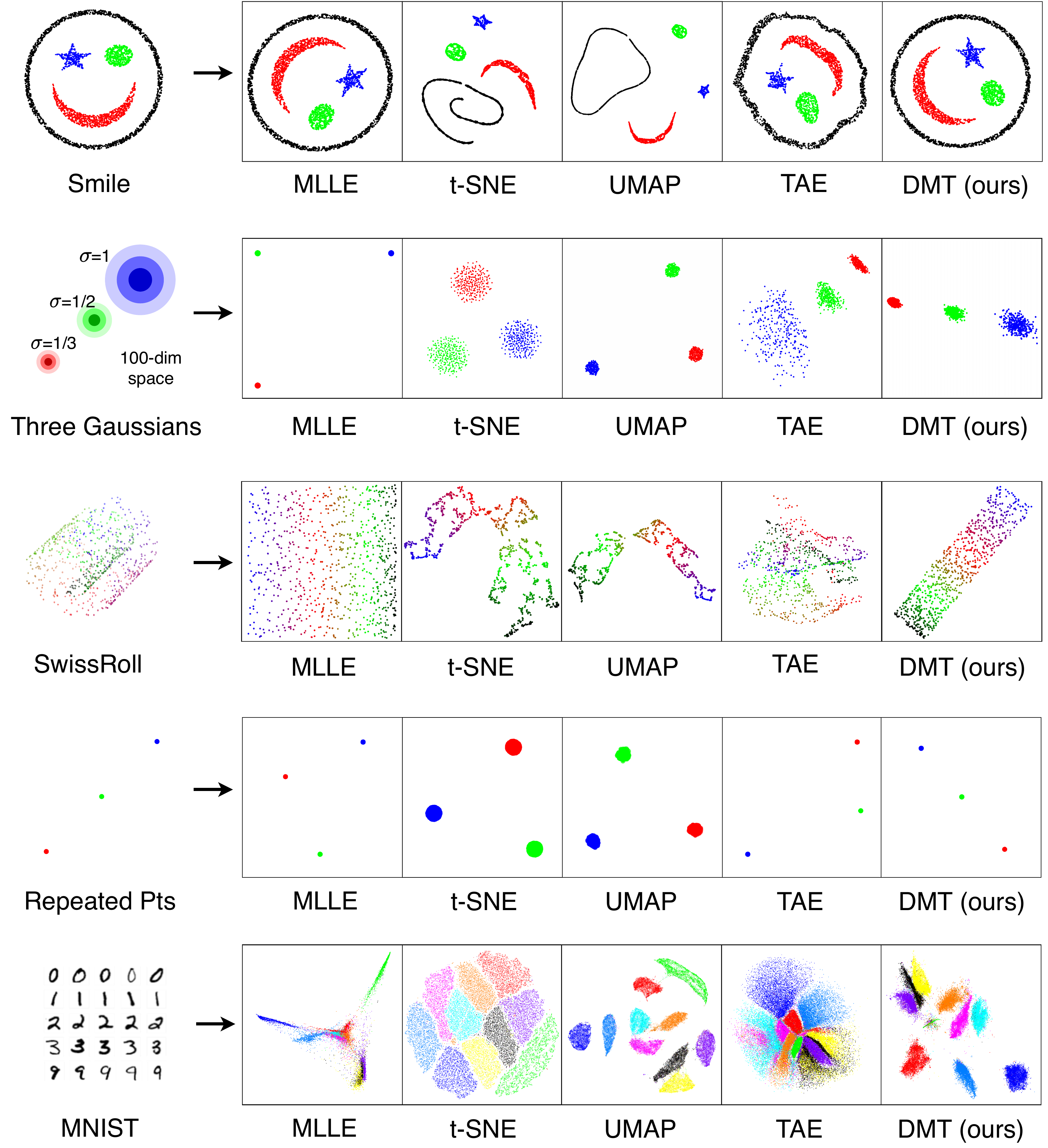}}
  \caption{Problems with existing methods. 
  (Row 1) t-SNE and UMAP fail to preserve the local- and global-geometrical structure of the 2-D smiling face. 
  (Row 2)  MLLE, t-SNE, and UMAP fail to preserve the distributional structure of the three Gaussians. 
  (Row 3) t-SNE, UMAP, and TAE fail to unfold the 3-D SwissRoll.
  (Row 4) At every location of the 3 points, there are 300 identical points. However, t-SNE and UMAP produce dispersed outputs, the reason being that they set a minimum distance threshold to separate every pair of points artificially, which can distort the geometry.
  (Row 5) MLLE fails on the MNIST data; the boundaries of the digit clusters should be fading away, but t-SNE and UMAP produce un-naturally sharp boundaries, and TAE results in even more un-natural distributions, distorting the distributional structure. In contrast, DMT produces the overall best results in preserving local and global geometry and sample distribution for all the datasets. A quantitative evaluation in Table~\ref{tab:ToyQuantitative} in the Experiments section further illustrate these advantages.}
 \label{fig:QualitativeComparison}
\end{figure}

In this paper, we propose a deep manifold learning framework, called {\em deep manifold transformation} ({\bf DMT}), to overcome the above problems in {\em unsupervised} NLDR and manifold learning. DMT is a deep learning method enhanced by {\em local geometry-preserving} ({\bf LGP}) constraints.  Minimizing LGP losses regularizes network training and encourages a network to be a one-to-one transformation between two metric spaces (layers). Such a DMT network maps points between training samples in the input space to interpolated points between corresponding points in the latent space. This property makes the transformation well-behaved to mitigate the problems above.

The main contributions of this paper are summarized as follows: 
\begin{itemize}
    \item {\bf The DMT Framework.} We propose the DMT framework for deep manifold learning. DMT uses cross-layer metric- or geometry-preserving constraints to obtain an embedding and an NLDR transformation. It is general enough to be applicable to most existing neural networks for regularizing their learning solutions. 
    
    \item {\bf LGP Formulations and DMT-encoder.} We propose LGP formulations based on cross-layer differences in pairwise distance or cross-layer divergence in pairwise similarity. Based on this, we develop DMT-encoder and DMT-autoecoder for DMT learning. DMT-encoder can accomplish unsupervised deep manifold learning for NLDR without a decoder. It also provides a transformation applicable to unseen data, in addition to producing an embedding. 
    
    \item {\bf Comparative Evaluation.} We provide extensive experimental results, with qualitative and quantitative comparisons and ablation study, to demonstrate significant advantages of DMT in comparison with popular leading algorithms. 
    

\end{itemize} 

In the following, Section~2 introduces the DMT framework, notation, and basic LGP concepts; Section~3 describes the cross-layer LGP constraint and losses and then DMT-encoder and DMT-autoencoder; Section~4 presents comparative experimental results and ablation study. 


\section{Deep Manifold Transformation}

DMT imposes LGP constraints on a deep neural network to achieve manifold learning-based NLDR. As a method for deep manifold learning, DMT gradually unfolds manifolds in the high-dimensional input space, layer by layer, onto regions in a latent Euclidean space such that Euclidean distance can be used reasonably to approximate the geodesic distance along a manifold in the input. The DMT framework is illustrated in Figure~\ref{fig:DMT+} with an autoencoder architecture consisting of an $L$-layer encoder and an $L$-layer decoder. Next, we introduce the notations, concepts, and local distance-based cross-layer LGP loss.

\begin{figure*}[!htb]
  \centering{\includegraphics[width=6.3in]{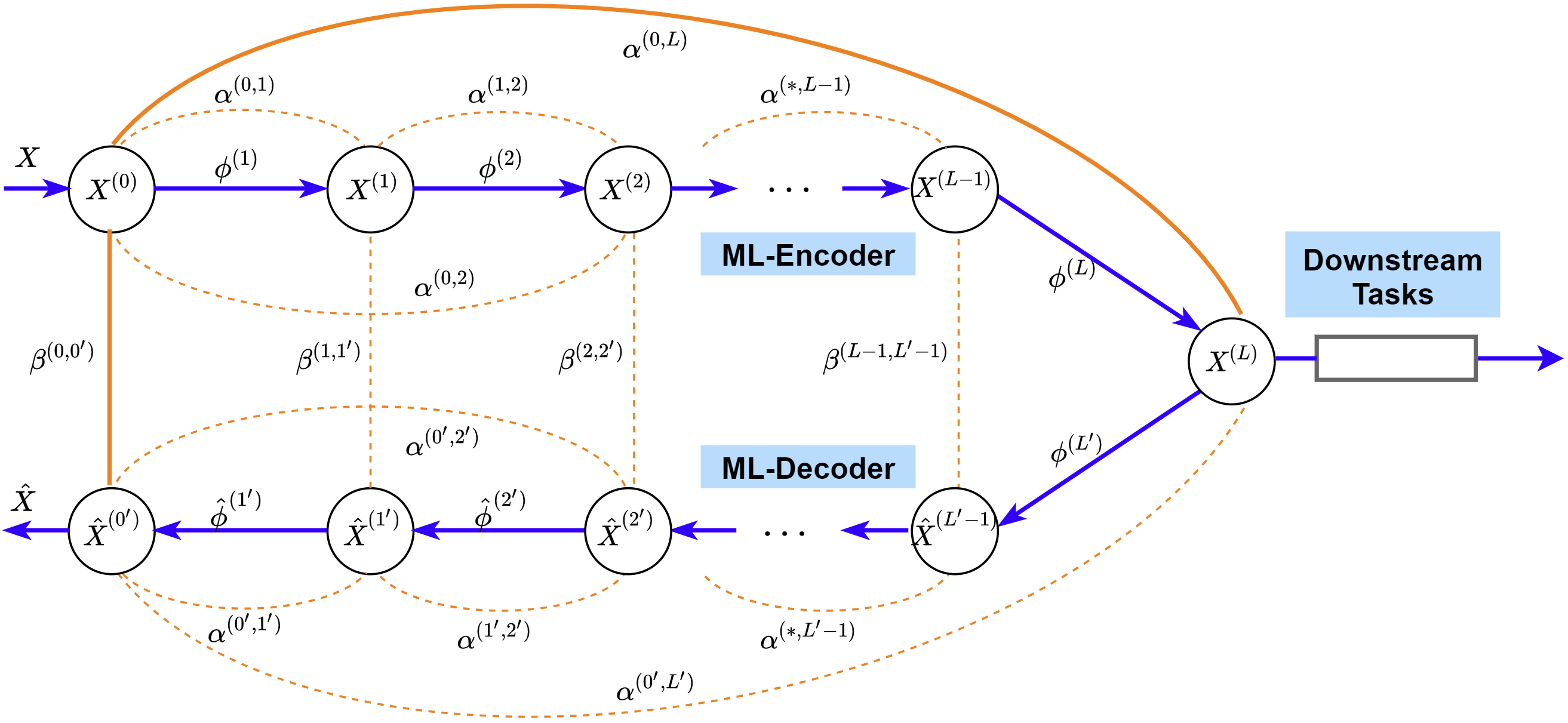}}
  \caption{Illustration of the DMT framework with cross-layer metric-preserving constraints (best viewed in color). The standard autoencoder consists of a cascade of transformations $\phi^{(l)}$ shown in the blue arrows, from the input $X=X^{(0)}$ to the latent layer $Z=X^{(L)}$ (NLDR) and then from $Z$ to the output layer $\hat X=X^{(0')}$ (data generation). It imposes a reconstruction loss between the input $X=X^{(0)}$ and output $\hat X=X^{(0')}$, a type of cross-layer {\em non-metric-preserving} constraint shown in the solid straight line in orange color with the weight $\beta^{(0,0')}$. The reconstruction loss may be defined as $\sum_i \|x^{(0)}_i-x^{(0')}_i\|^2$. A DMT-encoder imposes {\em local  geometry-preserving} (LGP) constraints across-layers $l$ and $l'$ on top of the standard encoder to restrict the transformations $\phi^{(l)}$ to satisfy LGP constraints as far as possible. The most essential cross-layer constraint among others is the one between the input ($l=0$) and latent ($l'=L$) layers, that is, between $X$ and $Z$ as shown in the solid arc in orange color with the weight $\alpha^{(0,L)}$. Other cross-layer constraints in dashed arcs with weights $\alpha^{(l,l')}$ may be included optionally. As mentioned earlier, a DMT-encoder can perform NLDR and embedding learning without a decoder. A DMT-decoder can be added to provide extra constraints in reconstruction errors weighted by $\beta^{(l,l')}$. }
 \label{fig:DMT+}
\end{figure*}

\subsection{Data, Metric Space and Graph}

Let $X=\left\{x_1,\ldots,x_M\right\}$ be a dataset of $M$ points in the input space $\R^N$ and $\cS=\left\{1,\ldots,M\right\}$ the index set. We assume that these points are samples on one or several manifolds $\cM_X\subset\R^N$, with the intrinsic dimensionality of each manifold being up to $n$. Such an $n$ is the lowest possible dimensionality for $Z=X^{(L)}$ to encode the information the manifolds losslessly.  Whereas manifolds are continuous hyper-surfaces, DMT works on discrete samples. Therefore, DMT needs to define a neighborhood system $\cN_X$ on $X$.


Without prior knowledge about $X$, we use Euclidean distance $d_X$ as the closeness measure for convenience. This constitutes a metric space $(X^{(l)},d^{(l)}_X)$ for each layer $l\in\{0,\cdots,L\}$. Let $d^{(l)}_{ij}=d_X(x^{(l)}_i,x^{(l)}_j)$, and the distance matrix be $D_X^{(l)}= [d^{(l)}_{ij}]_{i,j=1}^M$. (In the following, the subscript and superscript may be dropped to brief $d_X$ as $d$, $D^{(l)}_X$ as $D$, and $(X^{(l)},d^{(l)}_X)$ as $(X,d)$ without confusion in the context.) Although Euclidean distance $d^{(0)}_X$ in the input space is generally not an appropriate measure for nontrivial data analysis tasks, the goal of DMT is to make use of information contained in $(X^{(0)}, d_X^{(0)})$ as a start point to transform the manifolds represented by $X=X^{(0)}\subset \cM_X\subset\R^N$ nonlinearly onto regions in the lower dimensional latent space $(X^{(L)},d^{(L)}_X)$ such that Euclidean distance (or inner product) can provide a sensible distance measure therein.

Now, a graph $\cG(X^{(l)},D^{(l)}_X,\cN^{(l)}_X)$ can be constructed from each $(X^{(l)},d^{(l)}_X)$, as a discrete representation of $\cM_X$ at layer $l$. The neighborhood system for layer $l$ is $\cN^{(l)}_X=\{\cN^{(l)}_{i} \ | \ \forall i\in\mathcal{S}\}$ where $\cN^{(l)}_{i}$ is the set of neighbor indexes for $x^{(l)}_i$. There are two types: (1) {\em binarized neighborhood} defined by either the $k$-NN of $x^{(l)}_i$ or the neighbors within its $r$-ball; (2) {\em weighted neighborhood} composed of all $j\neq  i$ with a weight $u_{ij}\in(0,1]$. Generally for real world data, only a small number of $u_{ij}$ values are significantly nonzero. We hope as layer $l$ goes deeper,  $u^{(l)}_{ij}$ would become a more sensible proximity metric.

\subsection{Manifold Transformation on Graphs}

Manifold learning for NLDR with an LGP property finds a local homeomorphism
$\Phi : \cG\left(X,D_X,\cN_X\right)\rightarrow\ \cG\left(Z,D_Z,\cN_Z\right)$
which not only transforms from $X$ to $Z$ but also preserves local metrics in the two spaces. While $\Phi$ can be highly nonlinear and complex, it may be decomposed into  a cascade of $L$ less nonlinear, local homeomorphisms $\Phi=\ \phi^{(L)}\circ\cdots\circ\phi^{(2)}\circ {\phi}^{(1)}$. We use an $L$-layer DMT-encoder to achieve this. 

 
The layer-wise transformation can be written as
$
X^{(l+1)}=\phi^{(l)}\ (X^{(l)}, \ D^{(l)},\ \cN^{(l)}\ | \ W^{(l)})
$
where $W^{(l)}$ are the neural network's transformation matrices to be learned. A nonlinear activation follows this as usual. The updates for all $i\in\cS$ lead to local changes in graphs, namely, proper local deformations for unfolding $\cM^{(l)}_X$.  $D^{(l)}$ at the concerned layers are updated according to $X^{(l)}$. 

\subsection{Local Distance-Preserving Constraint}

LGP imposed metric discrepancy constraints across-layers.  One possible formulation is that based on isometry, which preserves local distances in a straightforward manner. We call it {\em locally isometric smoothness} (LIS). The LIS loss is introduced below as a precursor to the LGP loss.

Consider an effective transformation $\phi^{(l,l')}$ between any two layers $l$ and $l'$ along the cascade of transformations $\Phi$. The LIS loss requires that pairwise local distances be preserved by $\phi^{(l,l')}(W)$, and implements this by minimizing the following cross-layers loss
\begin{equation*}
  \begin{aligned}
 &\cL_{LIS}^{(l,l')}(W  \ | \ X^{(l)},X^{(l')}) = \\
 & \ \ \ \ \cL_{iso}(W  \ | \ X^{(l)},X^{(l')}) + \mu \cL_{push}(W  \ | \ X^{(l)},X^{(l')}),
  \end{aligned}
\end{equation*}

where $\mu\ge 0$ is a weight. The first term simply aims to preserve corresponding pairwise distances across layers
\begin{equation}
  \begin{aligned}
\cL_{iso}^{(l,l')}&(W  \ | \ X^{(l)},X^{(l')}) = \\
 & \sum_{i\in\cS}\sum_{j\in\cN_i} \alpha^{(l,l')}
  \left|d(x_i^{\left(l\right)},x_j^{(l)})-d(x_i^{(l')},x_j^{(l')})\right|,
  \end{aligned}
  \label{equ:loss-iso}
\end{equation}
where $\cN_i$ in the binary neighborhood system are determined using the $k$-NN or $r$-ball rule, $\alpha^{(l,l')}$ are the weights for the importance of the constraint across layers $l$ and $l'$. The second term is a "push-away" loss 
\begin{equation}
  \begin{aligned}
& \cL_{push}(W  \ | \ X^{(l)},X^{(l')}) = \\
 & \ \ \ \ -\sum_{i\in\cS}\sum_{{j\not\in\cN^{(l)}_i} \& {d(x_i^{(l')},x_j^{(l')})<B}}
 d(x_i^{(l')},x_j^{(l')}),
  \end{aligned}
\label{equ:loss-push}
\end{equation}
where $B$ is a distance threshold. This term makes pairs of $(x_i,x_j)$, which are non-neighbors ($j\not\in\cN^{(l)}_i$) at layer $l$ but nearby ($d(x_i^{(l')},x_j^{(l')})<B$) at layer $l'$, repel each other so as to unfold or flatten the manifolds. $\cL_{iso}$ and $\cL_{push}$ are equivalent to the ‘‘local stress’’ and ‘‘repulsion’’ in \cite{CHEN-JASA-2009}, respectively. The parameter $\mu$ starts from a positive value at the beginning of DMT learning to enable the auxiliary term $\cL_{push}$ and gradually decreases to $0$ so that finally only the real objective $\cL_{iso}$  takes effect. 

It can be shown that minimizing $\cL_{iso}$ leads to optimized bi-Lipschitz continuity (see Supplements~A.1).  Lipschitz continuity has been used to formulate regularizers for stabilizing neural networks \cite{Bartlett-NIPS-2017,Anil-ICML-2018,Weng-ICLR-2019,Cohen-ICML-2019,Zhou-Lip-GAN-ICML2019,Qi-LS-GAN-IJCV2020}. Those articles impose the Lipschitz constraint on either transformation matrices $W$ or gradient projection matrices at individual layer(s) rather than across layers as in DMT. Moreover,  those Lipschitz regularizers generally deviate optima of the original loss function. In contrast, the cross-layer LIS loss constitutes the target loss for local-distance preservation;  in the meantime, it also serves as a regularizer without deviating the target loss's objective because the target loss and the regularizer are the same thing. Thereby, LIS kills two birds with one stone.

\section{Local Geometry-Preserving Constraint}

LGP uses two strategies beyond LIS: (1) converting pairwise distances nonlinearly to {\em similarities} and using them as adaptive edge weights for the weighted neighborhood system, and (2) using a divergence-based loss in place of the straight distance difference-based loss of LIS. 
This section will formulate nonlinear conversion from distance to similarity, the divergence-based cross-layer LGP loss, and DMT-encoder and DMT-autoencoder.

\subsection{Distance-Similarity Conversion}

The raw $d(x_i,x_j)$ at the input layer is calibrated locally about $x_i$ into $d_{j|i}\defeq d(x_i,x_j)-\rho_i \ge 0$
where $\rho_i = \min_j \{d(x_i,x_j) \ | \ j\in \cN_i\}$ is the nearest neighbor distance. Then, $d_{j|i}$ are converted to { similarities} $u_{j|i}=g(d_{j|i}) \in (0,1]$ where $g$ is a {\em nonlinear}, monotonically decreasing function of $d_{j|i}\in \R_{\ge 0}$.

More specifically, $u_{j|i}(\sigma,\nu) = g(d_{j|i} \ | \ \sigma,\nu)$ is the strength of the directional edge $i\to j$, in which $\sigma$ is the scale parameter for all $x_i$ and $\nu\in\R^+$ controls the shape in a similar way to the degree of freedom (DoF) in the $t$-distribution. In this paper, we use the following normalized squared $t$-distribution
\begin{equation}
  \begin{aligned}
  g(d_{j|i} \ | \ \sigma,\nu) 
  = C_\nu  \left(1+\frac{d_{j|i}^{2}}{\sigma \ \nu } \right)^{-(\nu+1)},
  \end{aligned}
  \label{equ:t2-distribution}
\end{equation}
in which
\begin{equation}
C_\nu  = 2\pi \left(\frac{\Gamma\left(\frac{\nu+1}{2}\right)}
{\sqrt{\nu \pi} \Gamma\left(\frac{\nu}{2}\right)}\right)^2,
\label{equ:t2-nu}
\end{equation}
is the normalizing function of $\nu$ which sets the limit $\lim_{\nu\to +\infty} g(0 \ | \ \sigma, \nu)=1$. The scaling parameter $\sigma>0$ is estimated from the data by best fitting to the following the equation
$
\sum_{j\neq i} u_{j|i}(\sigma,\nu)=\log_2 Q
$
for a given perplexity-like hyper-parameter $Q$. 



Since $u_{j|i}\neq u_{i|j}$ is generally non-symmetric, symmetrization is performed as $u_{ij}=u_{j  |  i}+u_{i | j} - u_{j  |  i} u_{i | j}$. The similarities $u_{ij}$ are confined in the range $(0,1]$ to form a fuzzy set but their sum over $i$ and $j$ needs not to be 1 to be a probability distribution. The complement set $1-u_{ij}$ can be considered as {\em dissimilarity}.  

The LGP loss is defined as the following what we call {\em two-way divergence}
\begin{equation}
  \begin{aligned}
    \cL&_{LGP}^{(l,l')} (W  \ | \ X^{(l)},X^{(l')}) = \\
  &  \sum_{i,j \in \cS, i \neq j}
    u^{(l)}_{i j} \log \frac{u^{(l)}_{i j}}{u^{(l')}_{i j}}
    + (1-u^{(l)}_{i j}) 
    \log 
    \frac{ 1-u^{(l)}_{i j}} {1-u^{(l')}_{i j}} .
  \end{aligned} 
  \label{equ:loss-LGP}
\end{equation}

The above formulation is also known as the fuzzy information for discrimination \cite{Bhandari-IS-1993} and the fuzzy set cross-entropy in UMAP \cite{mcinnes2018umap}. 

The first term on the right-hand side of Equ.~(\ref{equ:loss-LGP}) measures divergence between the fuzzy similarity sets $u^{(l)}_{i j}$ and $u^{(l')}_{i j}$ and imposes attraction forces between nearby (intra-manifold) point pairs $(i,j)$. It could be considered as a nonlinear divergence version of Equ.~(\ref{equ:loss-iso}), namely, the difference $\log u^{(l)}_{i j} - \log u^{(l')}_{i j}$ weighted by the similarity $u^{(l)}_{i j}$. The second term exerts dissimilarity-based repulsion forces between far-away (inter-manifold) pairs and is an divergence version of Equ.~(\ref{equ:loss-push}), namely, the difference $\log(1-u^{(l)}_{i j}) - \log(1-u^{(l')}_{i j})$ weighted by the dissimilarity $1-u^{(l)}_{i j}$. Therefore, $\cL_{LGP}$ can be considered as an upgraded version of $\cL_{LIS}$ in that it renders DMT more nonlinearity and flexibility for non-isometric deformations. By using the LGP constraint, DMT transforms intra-manifold points  to a cluster in the latent space, mainly as the result of the first term, and pushes away inter-manifold point pairs from each other to different clusters, mainly due to the second term. 

\subsection{DMT-Encoder and DMT-Autoencoder}

{DMT-encoder} imposes cross-layer constraints on a conventional encoder, such an MLP, as the sole loss as well as the regularizer. A DMT-encoder has the following form
\begin{equation}
\cL_{Enc}(W)=\sum_{(l,l')}\alpha^{(l,l')} \cL_{LGP}(W  \ | \ X^{(l)},X^{(l')}),
\label{equ:loss-encoder}
\end{equation}
where $\alpha^{(l,l')}$ are the weights for pairwise losses (cf. Figure~\ref{fig:DMT+}). As mentioned earlier, DMT-encoder can learn an NLDR transformation unsupervisedly without the need for a decoder, a major difference from an autoencoder. The following is the DMT-encoder pseudo-code.

\begin{algorithm}[ht]
\caption{DMT\_Encoder}
\textbf{ Input }: Data:$X^{(0)}$, learning rate $lr$, epochs $E$, number of encoder layers $L$, Weight hyperparameter $\alpha$,  $\nu_{List}$, $Q$, \\
Calculate $d_{i|j}^{(0)}$ and $\sigma^{(0)}$\\
Calculate $u_{ij}^{(0)}$ with (\ref{equ:t2-distribution})\\
Initialize the neural network $\phi_{Enc}(\ \ \cdot\ \ |W_{Enc})$\\
While$\{$ $i=0$; $i<E$; $i$++$\}${\\
    \textcolor{white}{lllllll} $\nu \longleftarrow \nu_{List}[i]$ \\
    \textcolor{white}{lllllll} Calculate $L$ layer's embedding \\ 
    \textcolor{white}{lllllllll} $X^{(L)} \longleftarrow \phi_{Enc}(X^{(0)}|W_{Enc})$\\
    \textcolor{white}{lllllll} Calculate $u_{ij}^{(L)}$ with (\ref{equ:t2-distribution})\\
    \textcolor{white}{lllllll} Calculate the DMT losses, $ \mathcal{L}_{Enc}^{(0,L)} $ with (\ref{equ:loss-encoder})
    \\
    \textcolor{white}{lllllll} Update parameters: \\
    \textcolor{white}{lllllll} $W_{Enc} \longleftarrow W_{Enc} - lr \cdot \alpha \frac{ \partial \mathcal{L}_{Enc}^{(0,L)} }{\partial W_{Enc}}$
}
\end{algorithm}

The unsupervised DMT-encoder differs from self-supervised encoders \cite{Chen-SimClr-Arxiv-2020,He-MoCo-CVPR-2020} in their input data types. The latter ones are for image data where convolutional kernels can be applied to extract nonlinear invariants, whereas DMT works on non-image data. 

A DMT-encoder can be augmented into a {DMT-autoencoder}. Once trained, the decoder par of the {DMT-autoencoder} can be used to generate new data of the learned manifolds. The loss function is composed of two loss terms with the weight matrices $W=[W_{Enc},W_{Dec}]$:
\begin{equation}
    \cL_{AE}(W) = \cL_{Enc}(W_{Enc}) + \beta \cL_{Rec}(W_{Dec}),
\label{equ:lae}
\end{equation}
where the $\beta$ is the weight, and the reconstruction loss is $\cL_{Rec}(W)= \sum_{i=1}^{M}\parallel{x_i}^{(0)}-{x}_i^{(0')}\parallel^2$.

\section{Experiments}

\subsection{Experimental Setup and Performance Metricc}

The following experiments are aimed to evaluate DMT in comparison with other {\bf four algorithms}: MLLE \cite{zhang2007mlle}, t-SNE \cite{Maaten-tSNE-2014}, UMAP \cite{mcinnes2018umap},  and TAE \cite{Moor-TAE-ICML2020} in terms of numerical metrics and visualization. The results of ISOMAP, LLE, AE, and other related methods are not as appealing, so they are not included in the paper.


{\bf Nine datasets} are used. Four of them are toy datasets:
(1) SwissRoll (3-D),
(2) Smile Face (2-D),
(3) Three Gauss (100-D), and
(4) Repeat Points (100-D), and the other five are real-world datasets: 
(5) Coil20 (128$\times$128-D), 
(6) Coil100 (128$\times$128$\times$3-D) \cite{nene1996columbia100}, 
(7) MNIST (28$\times$28-D),
(8) FMNIST (28$\times$28-D) \cite{xiao2017fashion} and 
(9) CIFAR-3 (58$\times$58$\times$3-D), a subset of CIFAR-10 composed of classes 0, 4, and 8. 


{\bf Six performance metrics} are used for quantitative comparison: 
(1) Continuity (CON) \cite{venna2006visualizing} measures how well the $k$-NN of a point are preserved when going from the latent to the input space (layer); the larger, the better.
(2) Trustworthiness (TRU) \cite{venna2006visualizing} is similar to CON but going from the input to the latent space; the larger, the better.
(3) Mean relative rank error (RRE) \cite{lee2009quality} measures the average of changes in neighbor ranking between the two spaces; the smaller, the better.
(4) Distance Pearson Correlation (DPC) is the Pearson correlation coefficient between corresponding pairwise distances in the two spaces; the larger, the better.
(5) Scatter Rank Mismatch (SRM), designed by the authors, measures the total absolute difference between class scatter ranks in the two spaces, where the scatter of class $c$ is calculated as $\xi_c=\frac{1}{\#\cS_c} \sum_{i\in\cS_c}\|x_i-\bar x_c\|$ in which $\bar x_c$ is the class mean vector and $\|\cdot\|$ is the 2-norm of a vector; the smaller, the better.
(6) Accuracy (ACC) of an SVM linear classifier working in the latent space; the larger, the better.  Among these, CON, TRU and RRE measure local structures' consistency between the input and latent spaces, defined based on $k$-NN. It is set as $k=M/20$ where $M$ is the number of training samples. Their mathematical definitions are given in Supplements A.2.





The hyper-parameters are set mostly the same for the nine datasets: the network structure is set to be [-1, 600, 500, 400, 300, 200, 2] where -1 represents the input data's dimension, batchsize=1500, learning rate=0.001, the optimizer is set to Adam.  To obtain a stable embedding, we vary $\nu$  slowly from 0.001 to $\nu_{end}$ in an exponential rate, as is often done with the learning rate.
Dataset-specific hyper-parameters used for different datasets are provided in Table~\ref{tbl:hyper-paras}.
If a new dataset is used, the user only needs to adjust a small number of hyper-parameters, such as $Q$ and $\nu$. 

\begin{table}[!hbt]
\centering
\caption{Dataset-specific Hyper-parameters}
\begin{tabular}{@{}c|cc@{}}
\toprule
Dataset                & $\nu_{end}$    & $Q$  \\ \midrule
Smile Face             &  100           & 40  \\
Three Gauss            &  100           & 40  \\
Repeat Points          &  100           & 40  \\
Swiss Roll             &  100           & 40  \\
Coil20                 &  100           & 10  \\
Coil100                &  100           & 10  \\
MNIST                  &  0.001         & 20  \\ 
FMNIST                 &  0.001         & 20  \\ 
Cifar3                 &  0.001         & 10  \\ 
\bottomrule
\end{tabular}
\label{tbl:hyper-paras}
\end{table}




\subsection{Results on Toy Datasets}

While qualitative results on the four toy datasets are compared in Figure~\ref{fig:QualitativeComparison}, the corresponding quantitative comparison is provided in Table~\ref{tab:ToyQuantitative}. It shows that DMT achieves the best or equally best results as compared to the other four methods in all the three local neighborhood-based metrics RRE, CON and TRU.


\begin{table}[t]
  \centering
  \footnotesize
  \caption{Quantitative Comparison on Toy Datasets}
  \begin{tabular}{@{}cc|ccccc@{}}
    \toprule
                                                                              & Metr.  & MLLE           & t-SNE & UMAP  & TAE            & LAD            \\ \midrule
    \multirow{3}{*}{\begin{tabular}[c]{@{}c@{}}Swiss\\ Roll\end{tabular}}     & RRE    & 0.021          & 0.015 & 0.018 & 0.035          & \textbf{0.007} \\
                                                                              & CON    & 0.955          & 0.944 & 0.932 & 0.949          & \textbf{0.957} \\
                                                                              & TRU    & 0.965          & 0.989 & 0.989 & 0.964          & \textbf{0.999} \\\midrule
    \multirow{3}{*}{\begin{tabular}[c]{@{}c@{}}Smile\\ Face\end{tabular}}   & RRE    & 0.241          & 0.027 & 0.004 & 0.002          & \textbf{0.001} \\
                                                                              & CON    & 0.786          & 0.947 & 0.992 & 0.999          & \textbf{1.000} \\
                                                                              & TRU    & 0.715          & 0.872 & 0.995 & 0.999          & \textbf{1.000} \\\midrule
    \multirow{3}{*}{\begin{tabular}[c]{@{}c@{}}Three\\ Gauss\end{tabular}}   & RRE    & 0.179          & 0.107 & 0.122 & 0.137          & \textbf{0.097} \\
                                                                              & CON    & 0.858          & 0.998 & 0.998 & 0.890          & \textbf{1.000} \\
                                                                              & TRU    & 0.822          & 0.993 & 0.998 & 0.886          & \textbf{0.999} \\\midrule
    \multirow{3}{*}{\begin{tabular}[c]{@{}c@{}}Repeat\\ Points\end{tabular}}  & RRE    & \textbf{0.000} & 0.141 & 0.189 & \textbf{0.000} & \textbf{0.000} \\
                                                                              & CON    & \textbf{1.000} & 0.883 & 0.826 & \textbf{1.000} & \textbf{1.000} \\
                                                                              & TRU    & \textbf{1.000} & 0.874 & 0.861 & \textbf{1.000} & \textbf{1.000} \\\midrule
                                                                            \end{tabular}
    \label{tab:ToyQuantitative}
    \end{table}

\subsection{Results on Real-World Datasets}




Table~\ref{tbl:RealWorldQuantitative} provides an overall quantitative comparison of the 5 real-world datasets in terms of all 6 evaluation metrics. MLLE is not included because it runs out of time or memory for all datasets. DMT has an overall lower SRM (scattered rank mismatch)  than the other methods, suggesting that it has a better ability to preserve the distributional structure of data.  DMT has the highest ACC (accuracy) score on all five real-world datasets. This indicates that DMT delivers the best transformation for downstream tasks. 

\begin{table}[t]
  \centering
  \caption{Quantitative Comparison on Real-world Datasets}
  \begin{tabular}{@{}cc|llll@{}}
    \toprule
    \multicolumn{1}{l}{}      & Metr.   & t-SNE          & UMAP           & TAE            & DMT            \\ \midrule
    \multirow{6}{*}{Coil20}   & RRE     & 0.038          & 0.042          & 0.061          & \textbf{0.032} \\
                              & CON     & 0.867          & 0.849          & 0.898          & \textbf{0.928} \\
                              & TRU     & 0.925          & 0.922          & 0.877          & \textbf{0.913} \\
                              & DPC     & 0.210          & 0.177          & 0.633          & \textbf{0.419} \\
                              & SRM     & 0.470          & 0.380          & \textbf{0.240} & 0.350           \\
                              & ACC     & 0.890          & 0.870          & 0.858          & \textbf{0.902}\\ \midrule
    \multirow{6}{*}{Coil100}  & RRE     & 0.091          & 0.109          & 0.073          & \textbf{0.060} \\
                              & CON     & 0.748          & 0.665          & 0.861          & \textbf{0.881} \\
                              & TRU     & 0.756          & 0.738          & 0.848          & \textbf{0.844} \\
                              & DPC     & 0.172          & 0.080          & 0.539          & \textbf{0.584} \\
                              & SRM     & 0.496          & 0.396          & 0.219          & \textbf{0.108} \\
                              & ACC     & 0.895          & 0.807          & 0.336          & \textbf{0.907} \\\midrule
    \multirow{6}{*}{MNIST}    & RRE     & 0.089          & 0.077          & 0.076          & \textbf{0.073} \\
                              & CON     & 0.855          & 0.837          & \textbf{0.876} & 0.840          \\
                              & TRU     & 0.872          & 0.887          & 0.881          & \textbf{0.896} \\
                              & DPC     & 0.364          & 0.339          & \textbf{0.583} & 0.457          \\
                              & SRM     & 0.880          & 0.880          & 0.520          & \textbf{0.320} \\
                               & ACC     & 0.834          & 0.966          & 0.755          & \textbf{0.967}\\\midrule
    \multirow{6}{*}{FMNIST}   & RRE     & 0.041          & 0.038          & 0.037          & \textbf{0.035} \\
                              & CON     & 0.934          & 0.938          & \textbf{0.970} & 0.939          \\
                              & TRU     & 0.949          & 0.947          & 0.952          & \textbf{0.958} \\
                              & DPC     & 0.526          & 0.597          & \textbf{0.780} & 0.533          \\
                              & SRM     & 0.440          & 0.440          & \textbf{0.160} & 0.320         \\
                              & ACC     & 0.692          & 0.687          & 0.606          & \textbf{0.700}   \\\midrule
    \multirow{6}{*}{Cifar3}   & RRE     & 0.112          & 0.110          & 0.122          & \textbf{0.104} \\
                              & CON     & 0.905          & 0.911          & 0.915          & \textbf{0.929} \\
                              & TRU     & 0.856          & 0.859          & 0.846          & \textbf{0.874} \\
                              & DPC     & 0.752          & 0.761          & 0.812          & \textbf{0.850} \\
                               & SRM     & \textbf{0.000} & \textbf{0.000} & 0.889          & \textbf{0.000}\\
                               & ACC     & 0.628          & 0.631          & 0.614          & \textbf{0.629}\\ \bottomrule
    \end{tabular}
    \label{tbl:RealWorldQuantitative}
  \end{table}

\begin{figure}[!htb]
  \centering
  \includegraphics[width=1.0\linewidth]{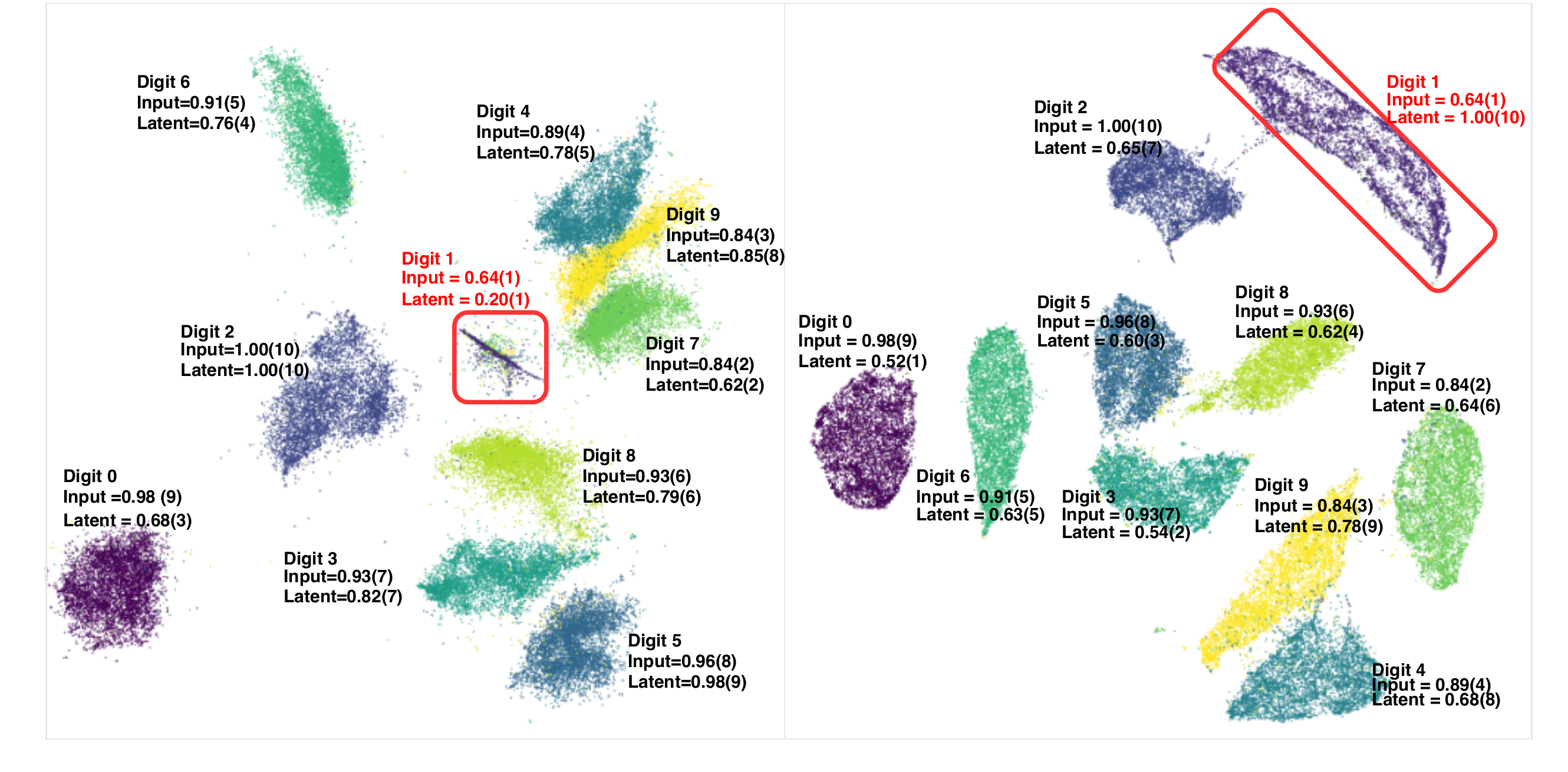}
    \caption{DMT (left) better preserves the distributional structure of data than UMAP (right).}
    \label{fig:MNIST-Distr-Structure}
\end{figure}

Figure~\ref{fig:MNIST-Distr-Structure} illustrates the ability of DMT to preserve the distributional structure of MNIST data compared to UMAP. The scatter values (the real numbers) and scatter ranks (in the brackets, the smallest being ranked \#1) of the 10 digits are computed from the data in the input space and latent space, respectively. We see that DMT better preserves the scatter ranks between the two spaces than UMAP. Taking a closer look at digit "1" which understandably has the smallest scatter of the ten digits, we see that DMT produces an embedding in which "1" has the smallest scatter (rank 1) of all the digits, whereas that of UMAP has the largest scatter (rank 10). 

Figure.~\ref{fig:Coil20-embeding-LISvsUMAP} compares the embeddings of Coil20 dataset produced by DMT and UMAP. It can be found that the embedding produced by DMT presents more details and integrity, while the embedding by UMAP presents crossovers and fragments. 

\begin{figure}[!htb]
  \centering
  \includegraphics[width=1.0\linewidth]{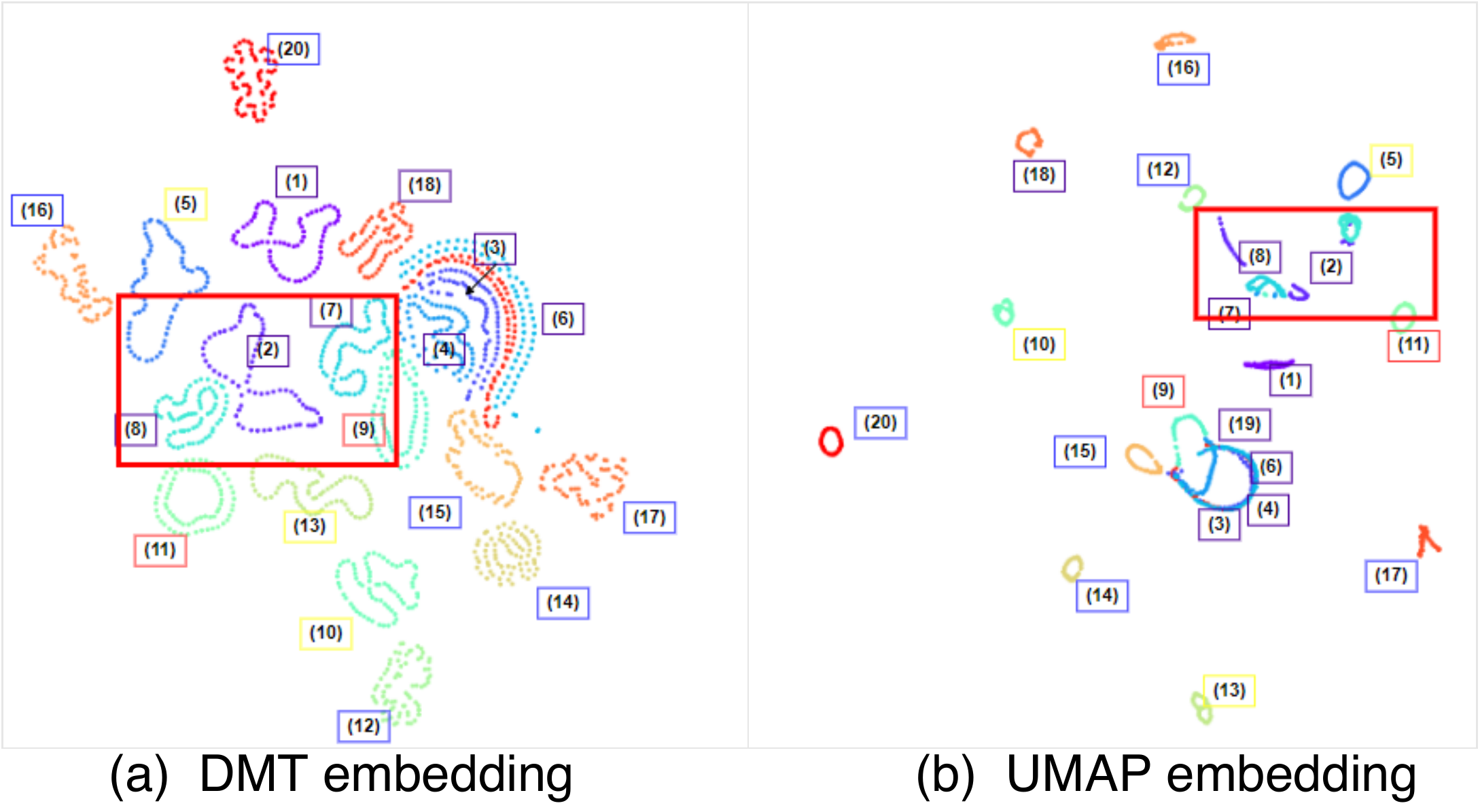}
    \caption{On the Coil20 dataset, the DMT ensures the integrity of its embedding whereas those of UMAP present crossover and tend to be fragmental.}
    \label{fig:Coil20-embeding-LISvsUMAP}
\end{figure}




\subsection{Manifold Data Generation}


A DMT-decoder can be learned together with DMT-autoencoder for generating new data of the learned manifold. Figure~\ref{fig:vis-inter-main} compares DMT and TAE in manifold interpolation and image generation (reconstruction). Embeddings learned by DMT are nicely organized, whereas those by TAE contains many defects. The images generated by DMT present sharper boundaries for the following reasons: (1) Contours in DMT embeddings do not cross or overlap. Hence, there is little unwanted interference from manifolds of other objects. (2) DMT preserves more details in the learned manifolds, which makes it easier to train the DMT-decoder. In contrast, TAE does not have such a good foundation in latent point interpolation and image generation.

\begin{figure}[!htb]
  \centering
  \includegraphics[width=1.0\linewidth]{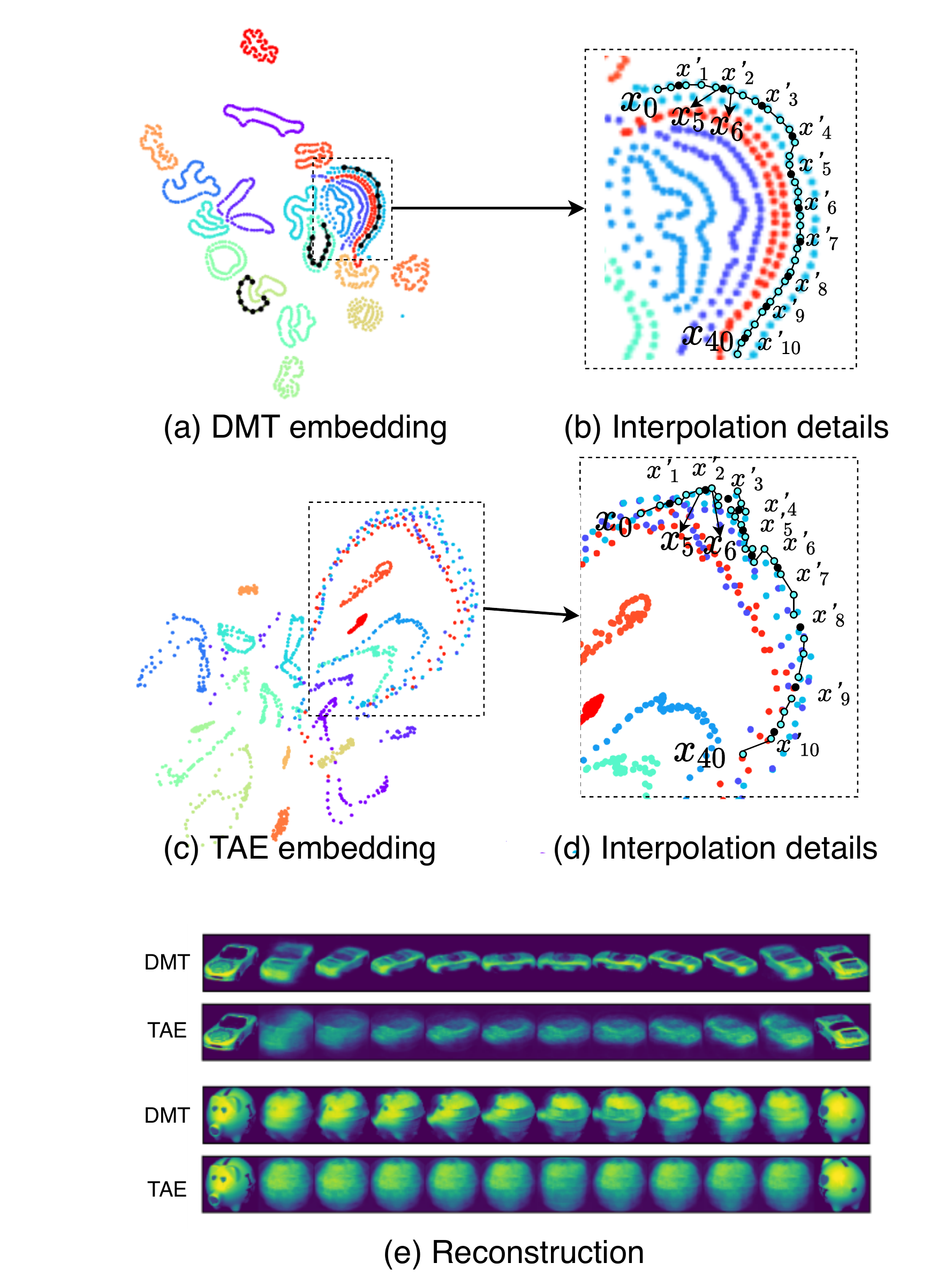}
    \caption{Coil20 images generated and interpolated by DMT-decoder vs. TAE-decoder. (a) DMT embedding in 2-D. (b) Interpolation: $x_0$ and $x_{40}$ are chosen as the two end-points in the embedding; 10 points are sampled equi-distance apart along the embedding. (c) \& (d) Interpolated samples and images generated by TAE. (e) Images reconstructed from the sample points. }
    \label{fig:vis-inter-main}
\end{figure}

\begin{table*}[t]
\centering
\caption{Summary of DMT Advantages}
\begin{tabular}{lccccc}
\toprule
                                    & MLLE  & t-SNE & UMAP  & TAE           & DMT  \\ \hline
Work for toy data                   & Yes   & No    & No    & Yes           & Yes         \\ \hline
Work for real-world data            & No    & Yes   & Yes   & Sometimes     & Yes  
    \\ \hline
Avoid one-to-many mapping           & Yes   & No    & No    & Yes           & Yes 
    \\ \hline
Preserve geometric structure        & No    & No    & No    & Sometimes     & Yes         \\ \hline
Preserve distributional structure   & No    & No    & No    & No            & Yes         \\ \hline
Applicable to unseen data           & No    & No    & No    & Yes           & Yes         \\ \bottomrule
\end{tabular}
\label{tbl:DMTadvantages}
\end{table*}

\subsection{DMT Performance Analysis}

Here, we analyze and summarise the performance of DMT based on the above experimental results. First, as shown in Figure~\ref{fig:QualitativeComparison}, only DMT worked successfully for all the nine datasets, whereas MLLE fails on real-world data, t-SNE and UMAP fail on toy data, and TAE does not perform as well as DMT on real-world data. Besides, t-SNE and UMAP can suffer from one-to-many mapping due to the direct optimization of embeddings rather than network parameters. Second, as shown in Table~\ref{tbl:RealWorldQuantitative}, only DMT achieves consistently good results across all the datasets in both local and global geometry-based metrics, while TAE sometimes fails. Third, DMT has an overall lower SRM than the other methods, which indicates a better ability to preserve the distributional structure of the data and may help detect outliers not identified by t-SNE and UMAP. Finally, it can be seen from Table~\ref{tbl:RealWorldQuantitative} that DMT has the highest ACC (accuracy) scores in all the five real-world datasets, providing the best NLDR transformation for downstream tasks.
The good characteristics are summarized and compared in Table~\ref{tbl:DMTadvantages}.

\vspace{2em}

\subsection{Ablation Study}


The hyper-parameter $Q$ can help DMT to make a sufficient balance between local and global structure preservation. We explore the effect of varying the perplexity $Q$ on the computation of pairwise similarity in the latent space on two datasets (Coil20 and MNIST). The resulting embedding results are shown in Figure~\ref{fig:ablation}. When $Q=5$, a relatively small $\sigma$ is obtained by binary search, which results in a smaller $u'_{ij}$ for the latent space embeddings will show more local information. As $Q$ increases, the $\sigma$ obtained becomes larger, so $u'_{ij}$ becomes larger, and the embedding will present structural information more globally.

\begin{figure}[!htb]
  \centering
  \includegraphics[width=3.2in]{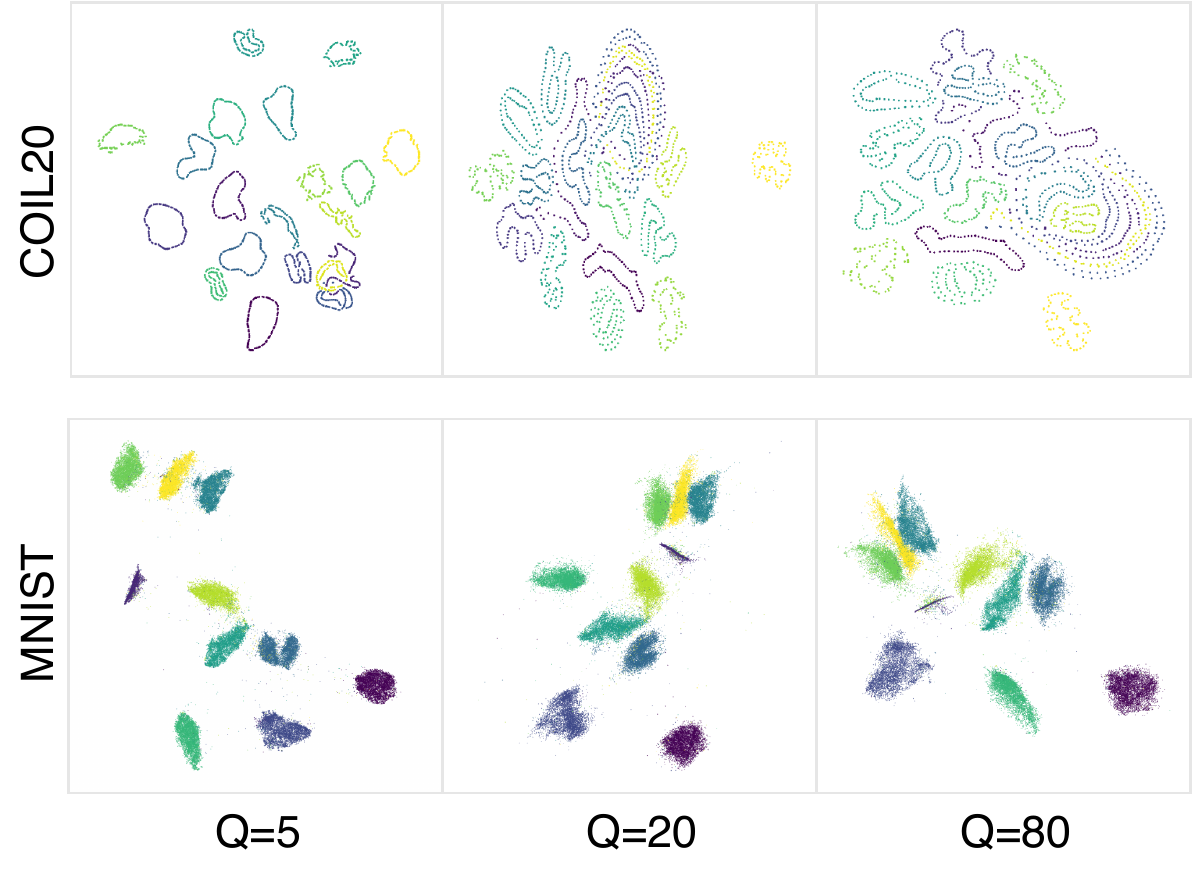}
    \caption{DMT embeddings learned with different perplexity $Q$ values.}
    \label{fig:ablation}
\end{figure}

\subsection{Complexity and Running Times}

The DMT loss is computed over all point pairs according to Equ.~(\ref{equ:loss-LGP}). 
The time complexity of DMT consists of two main parts: the initialization part and the model training part. For the initialization part, we estimate the $k$-NN for each node, with a computational complexity of $O(M^{1.14})$. In the model training part, we use a mini-batch  training scheme, with a computational complexity of $O(MM')$, where ${M'}$ is the batch size. The total computational complexity is lower than $O(M^2)$, enabling DMT-training on very large datasets.

Table~\ref{tbl:RunningTimes} shows a comparison of running times. As DMT is a neural network-based method, it is much more computationally expensive than t-SNE and UMAP if running on a CPU. However, due to GPU acceleration, the training time of DMT running on a V100 GPU is comparable to the learning time of UMAP running on Intel(R) Xeon(R) Gold 6248R (12 cores).


\begin{table}[]
  \caption{Running Times (in Seconds)}
  \centering
  \begin{tabular}{@{}cccccc@{}}
  \toprule
              & t-SNE & UMAP  & TAE  & DMT    \\ \midrule
  Coil20      & 22    & 12    & 82   & 25     \\
  Coil100     & 810   & 85    & 3197 & 128    \\
  MNIST       & 1450  & 87    & 1289 & 59     \\
  FMNIST      & 934   & 65    & 1173 & 58     \\
  GoogleNews  & 16906 & 361   & 5392 & 347    \\ \bottomrule
  \end{tabular}
  \label{tbl:RunningTimes}
\end{table}

\vspace{2em}
\section{Conclusion and Future Work} 

The DMT framework imposes the LGP constraints across neural network layers to constrain deep learning solutions in a local metric-preserving subspace of the original solution space. The cross-layer LGP constraints constitute the primary loss for DMT and also serve as neural transformation regularizers. This leads to an LGP-regularized solution to deep manifold learning. Extensive experiments demonstrate 
significant advantages of DMT over the other leading SOTA methods in preserving local- and global-geometric structure and distributional structure of data.



Future work includes optimizing the computation method in DMT learning.  Another direction is to apply DMT to solving semi-supervised learning and graph learning problems.


\newtheorem{remark}{Remark}
\newtheorem{theorem}{Theorem}
\renewcommand\thefigure{S.\arabic{figure}}
\renewcommand\thetable{S.\arabic{table}}
\renewcommand\theequation{S.\arabic{equation}}
\setcounter{equation}{0}
\setcounter{table}{0}
\setcounter{figure}{0}

\section*{Supplements}

The supplements include an introduction to the basic principle of locally isometric smoothness (LIS), definitions of the metrics for comparative evaluation, and the code.

\begin{figure*}[hb]
  \centering
  \includegraphics[width=6.5in]{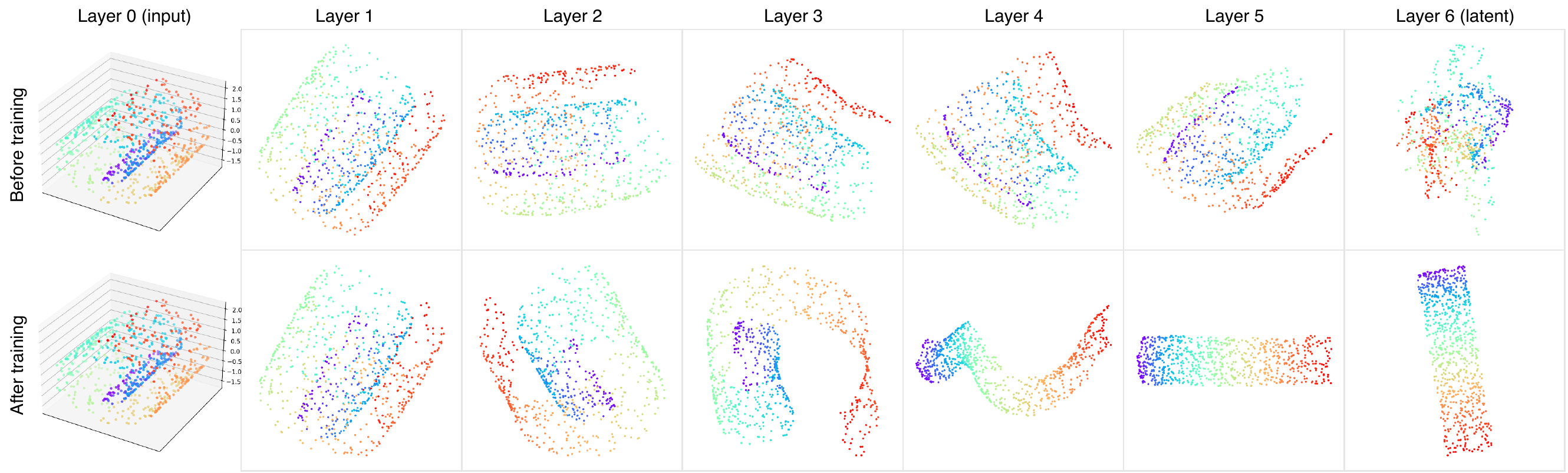}
    \caption{DMT unfolds Swiss-roll in 3-D, layer-by-layer, onto a 2-D planar region, visualized in 3-D.     The DMT network is initialized using the method of \cite{he2015delving}. The upper row shows the outputs of layers 1 through 6 of the initial network. The lower row shows the corresponding outputs when the DMT training is done, where PCA is used to project slight non-planar points onto the 2-D PCA plane. The hyper-parameters are set the same as the Swiss-roll experiments in the main paper.}
    \label{fig:swissroll_layerwise_vis}
\end{figure*}

\subsection*{A.1 Basic Idea of DMT}

This subsection presents two things. The first is an introduction to local bi-Lipschitz continuity as the basic principle used in LIS -- the precursor of DMT.  The second is a visualization example to show how a DMT network unfolds (or unrolls) a set of Swiss-roll data points in the 3-D input space onto a rectangular plane in 2-D latent space in a gradual, layer by layer manner.

{\bf A.1.1 $\mathbf{\cL_{iso}}$ and Optimal Locally Bi-Lipschitz Continuity}
\newtheorem{definition}{Definition}
\begin{definition}
Given two metric spaces $(X, d_X)$ and $(Y, d_Y)$, where $d_X$ denotes the metric on set $X$ and $d_Y$ the metric on $Y$. We say that a function $\Phi : X\rightarrow Z$ is {\em locally bi-Lipschitz continuous} if there exists a real constant $K \ge 1$ such that, for all $x_i$ and $\{x_j|j \in\mathcal{N}_{i}\}$, 
\begin{align*}
{1 \over K} d_X(x_i,x_j) \le d_Z(\Phi(x_i), \Phi(x_j))\le K d_X(x_i,x_j) 
\label{equ:math-bi-Lipschitz}.
\end{align*}
Any such $K$ is referred to as a {\em locally bi-Lipschitz constant} for the function $\Phi$. The smallest constant $K$ is called the (optimal) locally bi-Lipschitz constant.
\end{definition}
In the following, we denote $d_X=d_Y=d$ to simplify the mathematical notations without loss of generality. 

\begin{remark}
To satisfy the inequalities in Definition 1, there has to be $K \ge 1/K$ and $K \ge 1$. Therefore, for any function $\Phi$, the best possible bi-Lipschitz constant, or {\em the lower-bound} is $K=1$, obtained when $d(\Phi(x_i), \Phi(x_j)) = d(x_i,x_j)$. The closer $K$ is to 1, the better the locally bi-Lipschitz continuity is satisfied.
\end{remark}
 
Now let us examine the relationship between a cross-layer transformation $\Phi^{(l,l')}: X^{(l)}\rightarrow X^{(l')}$, its loss function 
(a component of $L_{iso}$ in Equ.~(1)) 
\begin{equation*}
  \begin{aligned}
\cL_{iso}&(\Phi^{(l,l')}) = \cL_{iso}^{(l,l')}(W  \ | \ X^{(l)},X^{(l')}) \\
 = & \sum_{i\in\cS}\sum_{j\in\cN_i} 
  \left|d(x_i^{(l)},x_j^{(l)})-d(x_i^{(l')},x_j^{(l')})\right|\\
 = & \sum_{i\in \cS}\sum_{j\in\cN_i}
  \left|d(x_i^{(l)},x_j^{(l)})-d\left(\Phi^{(l,l')}(x_i^{(l)}), \Phi^{(l,l')}(x_j^{(l)})\right) \right|
  \end{aligned}
\end{equation*}
and its bi-Lipschitz constant $K$. We have the following theorem.

\begin{theorem} 
The bi-Lipschitz constant $K$ for $\Phi^{(l,l')}$ reaches the lowest possible bound of 1 when the loss $\cL_{iso}(\Phi^{(l,l')})$ achieves the lowest possible value of 0.
\end{theorem}

\begin{proof}
When $\cL_{iso}(\Phi^{(l,l')})=0$, we have
\begin{equation*}
  \begin{aligned}
\ \ & \cL_{iso}(\Phi^{(l,l')}) = 
  \sum_{i\in\cS}\sum_{j\in\cN_i} 
  \left|d(x_i^{(l)},x_j^{(l)})-d(x_i^{(l')},x_j^{(l')})\right|\\
 & = \sum_{i\in \cS}\sum_{j\in\cN_i}
  \left|d(x_i^{(l)},x_j^{(l)})-d\left(\Phi^{(l,l')}(x_i^{(l)}), \Phi^{(l,l')}(x_j^{(l)})\right) \right| \\
  & = 0
  \end{aligned}
\end{equation*}
This requires that the following should be satisfied for all $i$ and $j \in\mathcal{N}_{i}$
\begin{equation*}
  \begin{aligned}
d\left(\Phi^{(l,l')}(x_i^{(l)}), \Phi^{(l,l')}(x_j^{(l)})\right) = d(x_i^{(l)},x_j^{(l)})
  \end{aligned}
\end{equation*}
This is the case of bi-Lipschitz continuity with $K=1$.

\end{proof}

The {\em auxiliary loss} $L_{push}$ of Equ.~(2) is used to mitigate the local optima problem in learning $\cL_{iso}$-based DMT networks (LIS networks). It ``pushes away" from each other, those pairs of points which are non-neighbors but are close to each other. It helps to avoid $L_{iso}$ from falling into local optima. In incorporating, we adopt a continuation strategy by varying the weight $\mu$ for $L_{push}$; that is, let $\mu$ start from a relatively significant positive value at the beginning of learning and then gradually decrease to 0, so that only $L_{iso}$ takes effect in the final stage.

{\bf A.1.2 DMT ``Unfolds'' Manifolds}

As a method for deep manifold learning, DMT unfolds manifolds in the high-dimensional input space gradually, layer by layer, onto regions in a latent Euclidean space such that Euclidean distance can be used reasonably to approximate the geodesic distance along a manifold in the input. We provide a visualization of the results of the layer-by-layer unfolding of Swiss-roll in Fig.~\ref{fig:swissroll_layerwise_vis}.

\subsection*{A.2 Definitions of Performance Metrics}

Six performance metrics are used, composed of two groups, are used for the comparative evaluation. Group 1 consists of five cross-layer metrics: 
\begin{itemize}
    \item[(1)] Continuity \cite{venna2006visualizing}  (the larger, the better),
    \item[(2)] Trustworthiness \cite{venna2006visualizing} (the larger, the better), 
    \item[(3)] Mean relative rank error \cite{lee2009quality} (the smaller, the better), 
    \item[(4)] Distance Pearson Correlation (the larger, the better),
    \item[(5)] Scatter Rank Mismatch, designed by this paper  (the smaller, the better).
\end{itemize} 
These are calculated between layer-pairs $(l,l')$. In the experiments presented in this paper, we used only one pair, that between $l=0$ (input space) and $l'=L$ (latent space). Adding other pairs may help improve but at the cost of more computation.  Group 2 is a global measure of
\begin{itemize}
    \item[(6)] Accuracy (the larger, the better)
\end{itemize} 
for a downstream supervised classification task. The definitions of the six metrics are given as follows.

\textbf{(1) Continuity (CON)} measures how well the k-NN of a point are preserved when going from the latent to the input space:
\begin{equation*}
\begin{aligned}
\text{CON} &=  1- \mathcal{T}_{\text{CON}}
      \sum_{i=1}^{M} \sum_{j \in \mathcal{N}_{i,k}^{(l)},j \not\in \mathcal{N}_{i,k}^{(l')}}(r^{(l')}_{i,j}-k),\\
      \mathcal{T}_{CON}&=\frac{2}{Mk  (2 M-3 k-1)},
\end{aligned}
\end{equation*}
where $\mathcal{T}_{\text{CON}}$ is a normalization term.  $r^{(l')}_{i,j}$ is the rank of $x^{(l')}_j$ in the $k$-NN of $x^{(l')}_i$. $M$ is the size of dataset. $\mathcal{N}_{i,k}^{(l')}$ is the set of indices to the $k$-NN of  $x^{(l')}_i$.  $k$ is set as $k=M/20$ where $M$ is the number of training samples.

\textbf{(2) Trustworthiness (TRU)} is similar to CON but going from the input to the latent space:
\begin{equation*}
  \begin{aligned}
    \text{TRU} &= 1-\mathcal{T}_{\text{TRU}}
    \sum_{i=1}^{M} \sum_{j \in \mathcal{N}_{i,k}^{(l')},j \not\in \mathcal{N}_{i,k}^{(l)}}(r^{(l)}_{i,j}-k),\\
    \mathcal{T}_{\text{TRU}} &=\frac{2}{Mk  (2 M-3 k-1)},
  \end{aligned}
\end{equation*}
where $\mathcal{T}_{\text{TRU}}$ is a normalization term. where $r^{(l)}_{i,j}$ is the rank of $x^{(l)}_j$ in the $k$-NN of $x^{(l)}_i$. $k$ is set as $k=M/20$ where $M$ is the number of training samples.

\textbf{(3) Mean relative rank error (RRE)} (Lee \& Verleysen, 2009) measures the average of changes in neighbor ranking between the two spaces: 
\begin{align*}
\text{RRE} = (\text{MR}^{(l,l')}_k + \text{MR}^{(l',l)}_k)/2,
\end{align*}
where $k_1$ and $k_2$ are the lower and upper bounds of the $k$-NN ($k_1=4$ and $k_2=10$ in this paper), and
\begin{align*}
    \text{MR}^{(l',l)}_k &= \mathcal{T}_{\text{RRE}} \sum_{i=1}^{M} \sum_{j \in 
    \mathcal{N}_{i,k}^{(l)}}\frac{|r^{(l)}_{i,j}-r^{(l')}_{i,j}|}{r^{(l)}_{i,j}},\\
     \text{MR}^{(l,l')}_k &= \mathcal{T}_{\text{RRE}} \sum_{i=1}^{M} \sum_{j \in \mathcal{N}_{i,k}^{(l')}}\frac{|r^{(l')}_{i,j}-r^{(l)}_{i,j}|}{r^{(l')}_{i,j})}),
\end{align*}
where $\mathcal{T}_{\text{RRE}}$ is the normalizing term 
\begin{align*}
\mathcal{T}_{\text{RRE}}=1/(M \sum_{k'=1}^{k} \frac{|M-2 k'|}{k'}) 
\end{align*}
in which $k$ is set as $k=M/20$ where $M$ is the number of training samples.

\textbf{(4) Distance Pearson Correlation (DPC)} is the Pearson correlation coefficient between corresponding pairwise distances in the two spaces:
\begin{equation*}
    \text{DPC}=\frac{1}{M-1}  \sum_{i,j\in \mathcal{S}, i\neq j} \left(\frac{D_{ij}^{(l)}-\bar{D^{(l)}}}{\sigma_{D^{(l)}}}\right)\left(\frac{D_{ij}^{(l')}-\bar{D^{(l')}}}{\sigma_{D^{(l')}}}\right)
\end{equation*}
where $D_{ij}^{(l)}$ is the Euclidean distance of $x^{(l)}_i$ and $x^{(l)}_j$. 
$D_{ij}^{(l')}$ is the Euclidean distance of $x^{(l')}_i$ and $x^{(l')}_j$. 
$\bar{D^{(l)}}$ and $\bar{D^{(l')}}$ are the mean value of $D^{(l)}$ and $D^{(l')}$. $\sigma_{D^{(l)}}$ and $\sigma_{D^{(l')}}$ are the variance of $D^{(l)}$ and $D^{(l')}$. 

\textbf{(5) Scatter Rank Mismatch (SRM)} designed by the authors, measures the total absolute difference between class scatter ranks in the two spaces, calculated as follows: First, the scatter of class $c$ is calculated 
$$\xi_c=\frac{1}{\#\cS_c} \sum_{i\in\cS_c}\|x_i-\bar x_c\|,$$
where $\bar x_c$ is the class mean vector; $\|\cdot\|$ is the 2-norm of a vector.
Then the SRM is defined as the normalized Spearman's footrule \cite{diaconis1977spearman}
$$\text{SRM}=\frac{1}{C^2}\sum_{c=1}^C | R^{(l)}_c-R^{(l')}_c|$$
where $R^{(l)}_c$ is the rank of $\xi_c^{(l)}$ in $\{\xi_0^{(l)}, \cdots, \xi_c^{(l)}, \cdots, \xi_C^{(l)}\}$, $|\cdot|$ is the absolute value, $C$ is the number of the classes. 

\textbf{(6) Accuracy (ACC)} of an SVM linear classification working in the learned latent space, calculated as follows:
(a) apply nonlinear dimensionality reduction method to obtain 2-D embedding results;
(b) use 5-fold cross-validation to evaluate the performance of the embedding results under the linear kernel SVM classifier.
(c) calculate the overall accuracy as the average of the five test accuracy numbers.


\subsection*{Code}

The code is included in the file Paper1616\_supp.zip.\\ \\

\bibliography{MLDL,vision}
\bibliographystyle{icml2021}


\end{document}